\documentclass{article} 
\usepackage{iclr2023_workshop,times}

\usepackage[utf8]{inputenc} 
\usepackage[T1]{fontenc}    
\usepackage{hyperref}       
\usepackage{url}            
\usepackage{longtable,booktabs}       
\usepackage{nicefrac}       
\usepackage{microtype}      
\usepackage{xcolor}         
\usepackage{tabularx}

\usepackage{mathtools,amssymb,amsthm, amsfonts}
\usepackage[nameinlink,capitalise]{cleveref}
\usepackage{tikz-cd,extarrows}
\usepackage{graphicx}
\usepackage{subcaption}

\usepackage[shortlabels,inline]{enumitem}
\setlist{nosep}
\usepackage{listings}
\hypersetup{
colorlinks=true,
linkcolor=black,
bookmarks=true,
bookmarksdepth=2
}
\usetikzlibrary{decorations.pathmorphing}
\usepackage{array}
\usepackage{nicefrac}

\newcommand{\RR}{\mathbb{R}}
\newcommand{\CC}{\mathbb{C}}
\newcommand{\NN}{\mathbb{N}}
\newcommand{\ZZ}{\mathbb{Z}}
\newcommand{\nrm}[1]{\lvert #1 \rvert}

\newcommand{\cA}{\mathcal{A}}
\newcommand{\cB}{\mathcal{B}}

\newcommand{\cP}{\mathcal{P}}
\newcommand{\cQ}{\mathcal{Q}}

\DeclareMathOperator{\Hom}{Hom}

\numberwithin{equation}{section}

\theoremstyle{plain}
\newtheorem{theorem}[equation]{Theorem}
\newtheorem{lemma}[equation]{Lemma}
\newtheorem{corollary}[equation]{Corollary}
\newtheorem{proposition}[equation]{Proposition}

\newtheorem{problem}[equation]{Problem}

\theoremstyle{definition}
\newtheorem{definition}[equation]{Definition}

\theoremstyle{remark}
\newtheorem{remark}[equation]{Remark}

\title{Fast computation of
permutation equivariant layers 
with the partition algebra
}

\author{Charles Godfrey,\(^1\) Michael G. Rawson,\(^1\) Davis Brown,\(^1\) and Henry Kvinge\(^1, ^2\)\\
\(^1\)Pacific Northwest National Laboratory, \(^2\) Department of Mathematics, University of Washington\\
\texttt{\{first\}.\{last\}@pnnl.gov} 
}

\newcommand{\todo}[1]{}
\newcommand{\hjk}[1]{}

\DeclareMathOperator{\tensordot}{dot}
\DeclareMathOperator{\set}{set}

\iclrfinalcopy 
\begin{document}

\maketitle

\begin{abstract}
Linear neural network layers that are either equivariant or invariant to permutations of their inputs form core building blocks of modern deep learning architectures. Examples include the layers of DeepSets, as well as linear layers occurring in attention blocks of transformers and some graph neural networks. The space of permutation equivariant linear layers can be identified as the invariant subspace of a certain symmetric group representation, and recent work parameterized this space by exhibiting a basis whose vectors are sums over orbits of standard basis elements with respect to the symmetric group action. A parameterization opens up the possibility of learning the weights of permutation equivariant linear layers via gradient descent. The space of permutation equivariant linear layers is a generalization of the partition algebra, an object first discovered in  statistical physics with deep connections to the representation theory of the symmetric group, and the basis described above generalizes the so-called orbit basis of the partition algebra. We exhibit an alternative basis, generalizing the diagram basis of the partition algebra, with computational benefits stemming from the fact that the tensors making up the basis are low rank in the sense that they naturally factorize into Kronecker products. Just as multiplication by a rank one matrix is far less expensive than multiplication by an arbitrary matrix, multiplication with these low rank tensors is far less expensive than multiplication with elements of the orbit basis. Finally, we describe an algorithm implementing multiplication with these basis elements.
\end{abstract}


\section{Introduction}

Invariance or equivariance to application-driven symmetry groups has served as
a guiding light for the design of neural network architectures for over two
decades, dating back at least to the introduction of convolutional networks \citep{fukushimaNeocognitronSelforganizingNeural1980}.
In the case where the underlying symmetries are permutations, several families
of architectures have appeared in the last five years: DeepSets \citep{zaheerDeepSets2018} (and its
successors), neural networks operating on graphs (where invariance to node
permutations is a natural desiderata) \citep{maronINVARIANTEQUIVARIANTGRAPH2019}, and transformers \citep{vaswaniAttentionAllYou2017a}. 

Permutation equivariant linear layers are linear maps \((\RR^n)^{\otimes m} \to (\RR^n)^{\otimes m'} \) 
equivariant with respect to the symmetric group \(\Sigma_n \) acting ``diagonally'' on tensors on the domain and target
(\cref{sec:permeqlayer} includes background and more detailed definitions).
These include linear layers used in DeepSets and some layers of transformers (for the later case see \citep{Kim2021TransformersGD}, \cite[\S 5.4]{bronsteinGeometricDeepLearning2021a}) as special
cases and as such are a fundamental building block for permutation equivariant/invariant
models. The result of \cite[Thm.
1]{maronINVARIANTEQUIVARIANTGRAPH2019}, which computes a basis for the vector
space parameterizing permutation equivariant linear layers, opens up the possibility of
\emph{learning} these layers, i.e. treating the parameters as
weights and optimizing them with some form of gradient descent.

There are good reasons related to computational efficiency  for avoiding the basis introduced in \citep{maronINVARIANTEQUIVARIANTGRAPH2019}: under a natural identification, permutation equivariant linear layers can be identified as tensors living in \((\RR^n)^{\otimes (m+m')}\). Just as multiplication by a rank one matrix of the form \(u v^T \), where \(u \) and \(v \) are vectors, can be executed far faster with sequential dot products than multiplication with an arbitrary matrix, contraction with a tensor in \((\RR^n)^{\otimes (m+m')}\) that decomposes as a Kronecker product, say \(u \otimes v \) where \(u \in (\RR^n)^{\otimes p}, v \in (\RR^n)^{\otimes p'} \) and \(p + p' = m+m' \), will generally be less expensive than contraction with an arbitrary tensor in \((\RR^n)^{\otimes (m+m')}\)
. Unfortunately, many of the basis vectors found in \cite[Thm. 1]{maronINVARIANTEQUIVARIANTGRAPH2019} lack such a decomposition (an example is given in \cref{sec:examples}).

Our main result, \cref{thm:alt-basis} below, exhibits an alternative basis for permutation equivariant linear layers in which all but one basis vector are explicitly constructed as Kronecker products. 
In this construction, we exploit the fact that the direct sum of the vector spaces of permutation equivariant linear layers, 
ranging from zero to infinity, forms an \emph{algebra} in which multiplication is the Kronecker product. 

In the case where \(m = m' \), this basis is known: permutation equivariant
layers \((\RR^n)^{\otimes m} \to (\RR^n)^{\otimes m} \) are exactly the
\emph{partition algebra}, an object first discovered in the context of
statistical mechanics in the mid 1990s
that has since been extensively studied
by mathematicians \citep{halversonPartitionAlgebras2005,benkartPartitionAlgebrasInvariant2017}
, and the basis of \cref{thm:alt-basis} is simply the ``diagram
basis'' of the partition algebra (that of \cite[Thm.
1]{maronINVARIANTEQUIVARIANTGRAPH2019} is known as the ``orbit sum'' basis). We
view this connection as a beautiful example of an idea that originated in
physics and made its way into deep learning architectures which are now being
applied to model physical phenomena, for example predicting molecular
geometry using the ZINC dataset \citep{zinc}.

For the sake of concreteness we work over \(\RR \) throughout, but all mathematical results in this paper apply over any field, including the complex numbers \(\CC \) and the finite field \(\ZZ/2\). This is not generality for generality's sake, as the latter is used in quantized neural networks.

The contributions of this paper include:
\begin{enumerate*}[(i)]
    \item an explicit construction of an alternative basis for the space of permutation equivariant linear layers,
    \item verification that in the case where the  permutation equivariant linear layers are the partition algebra, our basis coincides with the diagram basis, and
    \item outlining an algorithm implementing multiplication with these basis elements, and in the process showing they recover the efficient operations of \citep{panPermutationEquivariantLayers2022}.
\end{enumerate*}

\section{Related Work}
In \cref{sec:algorithm}, we show that the basis of \cref{thm:alt-basis} recovers the operations described in \cite[\S 4-5]{panPermutationEquivariantLayers2022}. Note however that while their work does prove that the set of ``sum/transfer/broadcast'' tensors, that they define, has the same cardinality as the basis of \cite[Thm. 1]{maronINVARIANTEQUIVARIANTGRAPH2019}, it does \emph{not} demonstrate that those ``sum/transfer/broadcast'' tensors are linearly independent nor that they span the space of permutation equivariant linear layers. 

By  \cite[Appendix]{panPermutationEquivariantLayers2022}, the basis of \cref{thm:alt-basis} also coincides with the one used in \cite[Appendix]{maronINVARIANTEQUIVARIANTGRAPH2019}. However, the latter authors also omit a proof of linear independence/spanning. 
The dimension of the space of permutation-equivariant linear layers grows extremely rapidly with \(m+m'\) and (for sufficiently large \(n \)) equals the Bell number \(B(m + m') \), see \cite[Thm. 1]{maronINVARIANTEQUIVARIANTGRAPH2019}. 
Thus while the case of \(m = m' =2 \) considered in their work may have been tractable as a one-off case, those of higher \(m + m' \) will not be. 

The partition algebra was first discovered (independently) in \citep{martin1991potts,jonesPottsModelSymmetric1994}. Its structure (in terms of generators and relations) and connections with representation theory of the symmetric group was identified in \citep{halversonPartitionAlgebras2005} (see also \citep{benkartPartitionAlgebrasInvariant2017} for more recent developments). In the case \(m = m' \), our \cref{thm:alt-basis} just gives an explicit description of the aforementioned diagram basis in terms of Kronecker products of diagonal tensors.

After completing this work, we became aware of the recent article \citep{pearce-crumpConnectingPermutationEquivariant2023}, which also points out the connection between permutation equivariant linear layers and partition algebras. This article has technical overlap with ours in \cref{sec:permeqlayer} (which overlap with multiple references for that matter), however it does not include our main results (those of \cref{sec:tensor-alg,sec:new-basis}), nor algorithms for multiplying with tensors representing permutation equivariant maps such as those of \cref{sec:algorithm}, nor connections with the operations of \cite[\S 4-5]{panPermutationEquivariantLayers2022}.

\section{Background} 
\label{sec:permeqlayer}

We fix a natural number \(n \in \NN \). In applications, this is the cardinality
of the set on which a permutation equivariant/invariant model operates. The vector space
\(\RR^n \) comes with a natural action of the symmetric group on \(n\) elements,
denoted \(\Sigma_n\), permuting the basis vectors \(\{e_1, \dots, e_n\}\):
\begin{equation}
    \sigma \cdot e_i = e_{\sigma(i)} 
    \text{  for  } \sigma \in \Sigma_n, i = 1, \dots, n.
\end{equation}
For each natural number \(m \in \NN\) the \(m\)-th tensor power 
\begin{equation}
    (\RR^n)^{\otimes m} = \underbrace{\RR^n \otimes \RR^n \otimes \cdots \otimes \RR^n}_{\text{\(m\) times}}
\end{equation}
is also a representation of \(\Sigma_n \) with the ``diagonal'' action, defined
for a \(\sigma \in \Sigma_n\) and basis tensor \(v_1 \otimes \cdots \otimes v_m
\in (\RR^n)^{\otimes m}\) as \(\sigma \cdot (v_1 \otimes \cdots \otimes v_m) := (\sigma \cdot v_1) \otimes \cdots \otimes (\sigma \cdot v_m)\).
\begin{remark}
    When \(m = 0 \) we adopt the convention
    \((\RR^n)^{\otimes m} = \RR \) with the trivial \(\Sigma_n\) action.
\end{remark}
We denote the vector space of \emph{all} (not necessarily equivariant) linear maps \((\RR^n)^{\otimes m} \to (\RR^n)^{\otimes m'} \) by \(\Hom((\RR^n)^{\otimes m}, (\RR^n)^{\otimes m'}) \). A linear map \( \varphi \in  \Hom((\RR^n)^{\otimes m}, (\RR^n)^{\otimes m'})  \) is \(\Sigma_n\)-equivariant if and only if \( \varphi (\sigma \cdot v) = \sigma \cdot \varphi(v) \) for all \( \sigma \in \Sigma_n, v \in (\RR^n)^{\otimes m} \), and we denote the vector space of such \(\Sigma_n\)-equivariant linear maps  \(\Hom((\RR^n)^{\otimes m}, (\RR^n)^{\otimes m'})^{\Sigma_n} \).
\begin{problem}
\label{item:param} 
    Parameterize \(\Hom((\RR^n)^{\otimes m}, (\RR^n)^{\otimes m'})^{\Sigma_n}  \), for
    example by giving a basis for it as a subspace of \(\Hom((\RR^n)^{\otimes m}, (\RR^n)^{\otimes m'}) \).
\end{problem}
The next two elementary results will be useful in what follows. 
Proofs are deferred to the appendix.
\begin{lemma}[{cf. \cite[\S 3]{halversonPartitionAlgebras2005}}]
    \label{lem:hom-tensor}
    \(\Hom((\RR^n)^{\otimes m}, (\RR^n)^{\otimes m'})^{\Sigma_n}\) is isomorphic
    to the \(\Sigma_n\)-invariant subspace of the tensor power
    \((\RR^n)^{\otimes (m+m')}\).
\end{lemma}
\begin{corollary}
    \label{cor:hom-tensor}
    The dimension of \( \Hom((\RR^n)^{\otimes m}, (\RR^n)^{\otimes
    m'})^{\Sigma_n}\) is a function of the \textbf{sum} \(m+ m'\) (for fixed \(n\)). Moreover, given any
    parametrization of the \(\Sigma_n\)-invariant subspace of \((\RR^n)^{\otimes
    l}\) for some \( l\in \NN \) there is a simple recipe to produce
    parametrizations of \(\Hom((\RR^n)^{\otimes m}, (\RR^n)^{\otimes
    m'})^{\Sigma_n}  \) for all \(m + m' = l\).
\end{corollary}

\Cref{lem:hom-tensor,cor:hom-tensor} reduce \cref{item:param} to the computation of the invariant subspace \(( (\RR^n)^{\otimes l}
)^{\Sigma_n}\). Now we give our notation for partitions. 
Given a tuple \((i_1, \dots, i_l) \in \{1, \dots, n\}^l \), define
subsets \( S(i_1, \dots, i_l)_j \subseteq \{1,\dots,l\} \) for \(j=1,\dots, n\) by 
\(S(i_1, \dots, i_l)_j = \{ k \in \{1, \dots, l\} \, | \, i_k = j \}\). By construction these subsets are pairwise disjoint and their union is \( \{1,\dots,l\}\); thus, they form a \emph{set partition of \( \{1,\dots,l\}\)}, which we will denote by 
\begin{equation}
    \Pi(i_1, \dots, i_l) = \{S(i_1, \dots, i_l)_j \, | \, j = 1, \dots, n\}.
\end{equation}
In this way we obtain a map from tuples \((i_1, \dots, i_l) \in \{1, \dots, n\}^l \) to set partitions \(\Pi(i_1, \dots, i_l)  \) of \( \{1,\dots,l\}\). The next lemma shows that the partition \(\Pi(i_1, \dots, i_l)  \) uniquely characterizes the \(\Sigma_n\)-orbit of \((i_1, \dots, i_l) \). For an illustration of the lemma in an explicit example, see \cref{sec:examples}.
\begin{lemma}[{cf. \cite[\S 1]{jonesPottsModelSymmetric1994},  \cite[\S 5.2]{benkartPartitionAlgebrasInvariant2017}}]
    \label{lem:bell-nos}
    The tuples \((i_1, \dots, i_l)\) and \((i'_1, \dots, i'_l)\)
    lie in the same \(\Sigma_n\) orbit if and only if they give rise to the same
    partition of \(\{1, \dots, l\}\), i.e. \(\Pi(i_1, \dots, i_l) = \Pi(i'_1, \dots, i'_l) \). 
\end{lemma}
The main theorem of \cite{maronINVARIANTEQUIVARIANTGRAPH2019} solves \cref{item:param} by exhibiting a basis described as follows: Let \( \cP \) be a fixed set partition of of \( \{1,\dots,l\}\), and define 
\begin{equation}
    e_{\cP} = \sum_{(i_1, \dots, i_l):\Pi(i_1, \dots, i_l) = \cP} e_{i_1} \otimes e_{i_2} \otimes \cdots \otimes e_{i_l} \in  \big((\RR^n)^{\otimes l}\big)^{\Sigma_n}.
\end{equation}
\begin{theorem}[{\cite[Thm. 1]{maronINVARIANTEQUIVARIANTGRAPH2019}}]
\label{thm:maron}
    The vectors \(\{e_{\cP} \, | \, \cP \text{ is a set partition of }  \{1,\dots,l\} \}\) form a basis of \(\big((\RR^n)^{\otimes l}\big)^{\Sigma_n}\).
\end{theorem}
In \cref{sec:invar-of-perm-reps} we provide intuition behind this construction.

\section{The partition algebra \(T(\RR^n)\) and its \(\Sigma_n\)-invariant subalgebra}
\label{sec:tensor-alg}
The invariant subspaces \(\big( (\RR^n)^{\otimes l}
\big)^{\Sigma_n}\) assemble to form an algebra. Multiplication in this algebra
can be interpreted as an operation that produces ``new permutation equivariant
layers from old'' by taking two permutation equivariant linear layers
\((\RR^n)^{\otimes m} \to (\RR^n)^{\otimes m'}\) and \( (\RR^n)^{\otimes p} \to
(\RR^n)^{\otimes p'}\)
and forming a third, \((\RR^n)^{\otimes (m+p)} \to (\RR^n)^{\otimes (m'+ p')}\),
by taking the tensor product. 


\begin{lemma}
\label{lem:algebra}
    For any \(p, p' \in \NN \) such that \(p + p' = l \), there is a natural bilinear map 
    \begin{equation}
        \big( (\RR^n)^{\otimes p} \big)^{\Sigma_n} \times \big( (\RR^n)^{\otimes p'} \big)^{\Sigma_n} \to \big( (\RR^n)^{\otimes l} \big)^{\Sigma_n} \text{  taking \( (v, w) \mapsto v \otimes w \).}
    \end{equation} 
\end{lemma}
The direct sum \(\bigoplus_{l\in \NN} \big( (\RR^n)^{\otimes l}
\big)^{\Sigma_n} \) forms an \(\RR\)-algebra since the bilinear maps of \cref{lem:algebra} are associative and unital in a suitable sense.

\section{A new-old basis for permutation equivariant linear layers}
\label{sec:new-basis}
For each partition \(\cP\) of \(\{1, \dots, l\}\), we let \(e_{\cP} \in \big(
(\RR^n)^{\otimes l} \big)^{\Sigma_n}\) be the basis element associated by
\cref{lem:orbit-sums,lem:bell-nos}. We claim there is another way to generate
basis elements from partitions, with potential computational advantages: given
\(\cP = \{S_1, \dots, S_n\}\), there is a multilinear map
\begin{equation}
    \label{eq:make-basis-elem}
    \Phi_{\cP}: \prod_{i=1}^n \Bigl((\bigotimes_{j = 1}^{\nrm{S_i}} \RR^n)^{\Sigma_n}\Bigl) \xrightarrow{\mu}  \left( (\RR^n)^{\otimes l} \right)^{\Sigma_n} \xrightarrow{\tau}
    \left( (\RR^n)^{\otimes l} \right)^{\Sigma_n}
\end{equation}
where \(\mu \) is obtained from repeated use of the bilinear maps from \cref{lem:algebra} and \(\tau \) is a permutation of \(\{1,\dots, l \} \) (acting on tensor factors)
whose inverse  arranges the sets \(S_i\) in
successive contiguous blocks.
\footnote{Explicitly, if \(S_i = 
\{s_{i1},\dots, s_{i \nrm{S_i}}\}\) then in \cref{eq:make-basis-elem} we can use
the permutation sending \((1,\dots, l) \mapsto (s_{11},\dots, s_{1\nrm{S_1}},
\dots, s_{n1},\dots, s_{n\nrm{S_n}})\)} 
Within each \((\bigotimes_{j = 1}^{\nrm{S_i}} \RR^n)^{\Sigma_n}\) there is a
basis element \(e_{\{\{1,\dots, \nrm{S_i}\}\}}\) corresponding to the partition with one set, i.e.
a diagonal tensor. We define: 
\begin{equation}
    d_{\cP} := \Phi_{\cP}(e_{\{\{1,\dots, \nrm{S_1}\}\}} , e_{\{\{1,\dots, \nrm{S_2}\}\}} , \cdots , e_{\{\{1,\dots, \nrm{S_n}\}\}}).
\end{equation}
\begin{remark}
    The choice of permutation of (labels of) \(\{1,\dots, l\}\) used to define \(\Phi_{\cP}\) is irrelevant, as any two
    choices differ by a sequence of permutations of the individual \(S_1,
    \dots, S_n\), and \(e_{\{S_i\}}\) is invariant to permutations of \(S_i\). 
\end{remark}
\begin{theorem}
    \label{thm:alt-basis}
    The vectors \(\Bigl\{ d_{\cP} \, | \, \cP \text{ is a set partition of } \{1, \dots, l\} \Bigl\} \)
    are a basis for \(\big( (\RR^n)^{\otimes l} \big)^{\Sigma_n}\).
\end{theorem}
As mentioned in the introduction, the space of permutation equivariant linear layers \((\RR^n)^{\otimes m} \to (\RR^n)^{\otimes m} \) is a partition algebra. 
\begin{proposition}
\label{prop:diagram}
In the \(m = m' \) case, the basis for \(\Hom((\RR^n)^{\otimes m}, (\RR^n)^{\otimes m})^{\Sigma_n}\) constructed in \cref{thm:alt-basis} coincides with the diagram basis of the partition algebra denoted by ``\( L_{\sim} \)'' in \citep{jonesPottsModelSymmetric1994}.
\end{proposition}

By construction, the tensors of \(d_{\cP} \) are factored Kronecker products. In
\cref{sec:algorithm}, we describe an algorithm for computing the \(\Sigma_n
\)-equivariant map \((\RR^n)^{\otimes m} \to (\RR^n)^{\otimes m'} \)
corresponding to multiplication with the tensor \(d_{\cP} \in (
(\RR^n)^{\otimes l} )^{\Sigma_n} \). While we leave an analysis of
computational cost (e.g. in terms of FLOPs) to future work, the algorithm shows
the factorization of \(d_{\cP} \) makes multiplying with them computationally
efficient (in particular more efficient than the the elements \(e_{\cP}\), see
\cref{sec:examples}). The description in \cref{sec:algorithm} also shows that
the \(d_{\cP} \) recover the ``sum/transfer/broadcast'' operations of \citep{panPermutationEquivariantLayers2022}.

\section{Conclusion and open questions}
\Cref{thm:alt-basis} provides a basis for permutation equivariant linear layers designed to be computationally efficient, since the tensors making up the basis are constructed as Kronecker products. One practical avenue for future work would be using the theorem to implement permutation equivariant linear layers for any user-specified \(m, m' \) (to the best of our knowledge, the implementations of our references hard-code paramterizations of \(\Hom((\RR^n)^{\otimes m}, (\RR^n)^{\otimes m'})^{\Sigma_n}  \) for specific values of \(m, m' \). Another direction would be to use a presentation of the partition algebra using a subset of the diagram basis as generators (see \cite[\S 3]{jonesPottsModelSymmetric1994}, \cite[Thm 1.11]{halversonPartitionAlgebras2005}) to obtain a relatively small but still expressive subspace of the permutation equivariant linear layers, as suggested in \cite[\S 6]{Kim2021TransformersGD}.



\nocite{*}
\bibliography{references}

\begin{thebibliography}{25}
\providecommand{\natexlab}[1]{#1}
\providecommand{\url}[1]{\texttt{#1}}
\expandafter\ifx\csname urlstyle\endcsname\relax
  \providecommand{\doi}[1]{doi: #1}\else
  \providecommand{\doi}{doi: \begingroup \urlstyle{rm}\Url}\fi

\bibitem[Benkart \& Halverson(2017)Benkart and
  Halverson]{benkartPartitionAlgebrasInvariant2017}
Georgia Benkart and Tom Halverson.
\newblock Partition {{Algebras}} and the {{Invariant Theory}} of the
  {{Symmetric Group}}, September 2017.

\bibitem[Bogatskiy et~al.(2022)Bogatskiy, Ganguly, Kipf, Kondor, Miller,
  Murnane, Offermann, Pettee, Shanahan, Shimmin, and
  Thais]{Bogatskiy2022SymmetryGE}
A~Bogatskiy, Sanmay Ganguly, Thomas Kipf, Risi Kondor, David~W. Miller, Daniel
  Murnane, Jan~T. Offermann, Mariel Pettee, Phiala~E. Shanahan, Chase~Owen
  Shimmin, and Savannah Thais.
\newblock Symmetry group equivariant architectures for physics.
\newblock \emph{ArXiv}, abs/2203.06153, 2022.

\bibitem[Bronstein et~al.(2021)Bronstein, Bruna, Cohen, and Veli{\v
  c}kovi{\'c}]{bronsteinGeometricDeepLearning2021a}
Michael~M. Bronstein, Joan Bruna, Taco Cohen, and Petar Veli{\v c}kovi{\'c}.
\newblock Geometric {{Deep Learning}}: {{Grids}}, {{Groups}}, {{Graphs}},
  {{Geodesics}}, and {{Gauges}}, May 2021.

\bibitem[Fukushima(1980)]{fukushimaNeocognitronSelforganizingNeural1980}
Kunihiko Fukushima.
\newblock Neocognitron: {{A}} self-organizing neural network model for a
  mechanism of pattern recognition unaffected by shift in position.
\newblock \emph{Biological Cybernetics}, 36\penalty0 (4):\penalty0 193--202,
  April 1980.
\newblock ISSN 1432-0770.
\newblock \doi{10.1007/BF00344251}.

\bibitem[Halverson \& Ram(2005)Halverson and
  Ram]{halversonPartitionAlgebras2005}
Tom Halverson and Arun Ram.
\newblock Partition algebras.
\newblock \emph{European Journal of Combinatorics}, 26\penalty0 (6):\penalty0
  869--921, August 2005.
\newblock ISSN 0195-6698.
\newblock \doi{10.1016/j.ejc.2004.06.005}.

\bibitem[Hartford et~al.(2018)Hartford, Graham, {Leyton-Brown}, and
  Ravanbakhsh]{hartfordDeepModelsInteractions2018}
Jason Hartford, Devon Graham, Kevin {Leyton-Brown}, and Siamak Ravanbakhsh.
\newblock Deep {{Models}} of {{Interactions Across Sets}}.
\newblock In \emph{Proceedings of the 35th {{International Conference}} on
  {{Machine Learning}}}, pp.\  1909--1918. {PMLR}, July 2018.

\bibitem[Jones(1994)]{jonesPottsModelSymmetric1994}
V.F.R. Jones.
\newblock The {{Potts}} model and the symmetric group.
\newblock \emph{Subfactors: Proceedings of the Taniguchi Symposium on Operator
  Algebras (Kyuzeso, 1993)}, 1994.

\bibitem[Kim et~al.(2021)Kim, Oh, and Hong]{Kim2021TransformersGD}
Jinwoo Kim, Saeyoon Oh, and Seunghoon Hong.
\newblock Transformers generalize deepsets and can be extended to graphs and
  hypergraphs.
\newblock In \emph{Neural Information Processing Systems}, 2021.

\bibitem[LeCun et~al.(1989)LeCun, Boser, Denker, Henderson, Howard, Hubbard,
  and Jackel]{cnns}
Y.~LeCun, B.~Boser, J.~S. Denker, D.~Henderson, R.~E. Howard, W.~Hubbard, and
  L.~D. Jackel.
\newblock Backpropagation applied to handwritten zip code recognition.
\newblock \emph{Neural Computation}, 1\penalty0 (4):\penalty0 541--551, 1989.
\newblock \doi{10.1162/neco.1989.1.4.541}.

\bibitem[Lim et~al.(2022)Lim, Robinson, Zhao, Smidt, Sra, Maron, and
  Jegelka]{limSignBasisInvariant2022}
Derek Lim, Joshua Robinson, Lingxiao Zhao, Tess Smidt, Suvrit Sra, Haggai
  Maron, and Stefanie Jegelka.
\newblock Sign and {{Basis Invariant Networks}} for {{Spectral Graph
  Representation Learning}}, September 2022.

\bibitem[Maron et~al.(2018)Maron, Ben-Hamu, Shamir, and
  Lipman]{maronINVARIANTEQUIVARIANTGRAPH2019}
Haggai Maron, Heli Ben-Hamu, Nadav Shamir, and Yaron Lipman.
\newblock Invariant and equivariant graph networks.
\newblock \emph{arXiv preprint arXiv:1812.09902}, 2018.

\bibitem[Martin(1991)]{martin1991potts}
P.P. Martin.
\newblock \emph{Potts Models and Related Problems in Statistical Mechanics}.
\newblock Series on Advances in Statistical Mechanics. {World Scientific},
  1991.
\newblock ISBN 978-981-02-0075-6.

\bibitem[Nikitin et~al.(2019)Nikitin, Tsilevich, and
  Vershik]{nikitinDecompositionTensorRepresentations2019}
P.~P. Nikitin, N.~V. Tsilevich, and A.~M. Vershik.
\newblock On the {{Decomposition}} of {{Tensor Representations}} of {{Symmetric
  Groups}}.
\newblock \emph{Algebras and Representation Theory}, 22\penalty0 (4):\penalty0
  895--908, August 2019.
\newblock ISSN 1572-9079.
\newblock \doi{10.1007/s10468-018-9804-6}.

\bibitem[Pan \& Kondor(2021)Pan and Kondor]{panFourierBasesSolving2021}
Horace Pan and Risi Kondor.
\newblock Fourier {{Bases}} for {{Solving Permutation Puzzles}}.
\newblock In \emph{Proceedings of {{The}} 24th {{International Conference}} on
  {{Artificial Intelligence}} and {{Statistics}}}, pp.\  172--180. {PMLR},
  March 2021.

\bibitem[Pan \& Kondor(2022)Pan and
  Kondor]{panPermutationEquivariantLayers2022}
Horace Pan and Risi Kondor.
\newblock Permutation {{Equivariant Layers}} for {{Higher Order Interactions}}.
\newblock In \emph{Proceedings of {{The}} 25th {{International Conference}} on
  {{Artificial Intelligence}} and {{Statistics}}}, pp.\  5987--6001. {PMLR},
  May 2022.

\bibitem[Paszke et~al.(2019)Paszke, Gross, Massa, Lerer, Bradbury, Chanan,
  Killeen, Lin, Gimelshein, Antiga, Desmaison, Kopf, Yang, DeVito, Raison,
  Tejani, Chilamkurthy, Steiner, Fang, Bai, and Chintala]{torch}
Adam Paszke, Sam Gross, Francisco Massa, Adam Lerer, James Bradbury, Gregory
  Chanan, Trevor Killeen, Zeming Lin, Natalia Gimelshein, Luca Antiga, Alban
  Desmaison, Andreas Kopf, Edward Yang, Zachary DeVito, Martin Raison, Alykhan
  Tejani, Sasank Chilamkurthy, Benoit Steiner, Lu~Fang, Junjie Bai, and Soumith
  Chintala.
\newblock Pytorch: An imperative style, high-performance deep learning library.
\newblock In H.~Wallach, H.~Larochelle, A.~Beygelzimer, F.~d\textquotesingle
  Alch\'{e}-Buc, E.~Fox, and R.~Garnett (eds.), \emph{Advances in Neural
  Information Processing Systems 32}, pp.\  8024--8035. Curran Associates,
  Inc., 2019.
\newblock URL
  \url{http://papers.neurips.cc/paper/9015-pytorch-an-imperative-style-high-performance-deep-learning-library.pdf}.

\bibitem[{Pearce-Crump}(2022)]{pearce-crumpBrauerGroupEquivariant2022}
Edward {Pearce-Crump}.
\newblock Brauer's {{Group Equivariant Neural Networks}}, December 2022.

\bibitem[{Pearce-Crump}(2023)]{pearce-crumpConnectingPermutationEquivariant2023}
Edward {Pearce-Crump}.
\newblock Connecting {{Permutation Equivariant Neural Networks}} and
  {{Partition Diagrams}}, January 2023.

\bibitem[Serre(1977)]{serreLinearRepresentationsFinite1977}
J.P. Serre.
\newblock \emph{Linear Representations of Finite Groups}.
\newblock Graduate Texts in Mathematics. {Springer New York}, 1977.
\newblock ISBN 978-1-4684-9458-7.

\bibitem[Sterling \& Irwin(2015)Sterling and Irwin]{zinc}
Teague Sterling and John~J. Irwin.
\newblock Zinc 15 – ligand discovery for everyone.
\newblock \emph{Journal of Chemical Information and Modeling}, 55\penalty0
  (11):\penalty0 2324--2337, 2015.
\newblock \doi{10.1021/acs.jcim.5b00559}.
\newblock URL \url{https://doi.org/10.1021/acs.jcim.5b00559}.
\newblock PMID: 26479676.

\bibitem[Thiede et~al.(2020)Thiede, Hy, and
  Kondor]{thiedeGeneralTheoryPermutation2020}
Erik~Henning Thiede, Truong~Son Hy, and Risi Kondor.
\newblock The general theory of permutation equivarant neural networks and
  higher order graph variational encoders, April 2020.

\bibitem[Vaswani et~al.(2017)Vaswani, Shazeer, Parmar, Uszkoreit, Jones, Gomez,
  Kaiser, and Polosukhin]{vaswaniAttentionAllYou2017a}
Ashish Vaswani, Noam Shazeer, Niki Parmar, Jakob Uszkoreit, Llion Jones,
  Aidan~N. Gomez, Lukasz Kaiser, and Illia Polosukhin.
\newblock Attention {{Is All You Need}}, December 2017.

\bibitem[Villar et~al.(2021)Villar, Hogg, Storey-Fisher, Yao, and
  Blum-Smith]{Villar2021ScalarsAU}
Soledad Villar, David~W. Hogg, Kate Storey-Fisher, Weichi Yao, and Ben
  Blum-Smith.
\newblock Scalars are universal: Equivariant machine learning, structured like
  classical physics.
\newblock In \emph{Neural Information Processing Systems}, 2021.

\bibitem[Wagstaff et~al.(2021)Wagstaff, Fuchs, Engelcke, Osborne, and
  Posner]{Wagstaff2021UniversalAO}
Edward Wagstaff, Fabian~B. Fuchs, Martin Engelcke, Michael~A. Osborne, and
  Ingmar Posner.
\newblock Universal approximation of functions on sets.
\newblock \emph{ArXiv}, abs/2107.01959, 2021.

\bibitem[Zaheer et~al.(2018)Zaheer, Kottur, Ravanbakhsh, Poczos, Salakhutdinov,
  and Smola]{zaheerDeepSets2018}
Manzil Zaheer, Satwik Kottur, Siamak Ravanbakhsh, Barnabas Poczos, Ruslan
  Salakhutdinov, and Alexander Smola.
\newblock Deep {{Sets}}, April 2018.

\end{thebibliography}
\bibliographystyle{iclr2023_workshop}

\appendix

\section{Invariants of permutation representations}
\label{sec:invar-of-perm-reps}

In this section, we give a concise proof of \cref{thm:maron} using a couple of elementary pieces of representation theory and a combinatorial lemma.
\begin{definition}
    Let \(G \) be a finite group and let \(X \) be a finite set with a
    \(G\)-action. The \textbf{associated permutation representation}, denoted
    \(\RR X \), is the vector space with one basis element \(e_x\) for each \(x
    \in X\) and the \(G \)-action defined by \(g \cdot e_x = e_{gx}\).
\end{definition}
The \(\Sigma_n\) representation \(\RR^n\) is a
permutation representation: \(\RR^n = \RR \{1, \dots, n\}\). In general, if \(X_1\) and \(X_2\) are finite sets with actions of a finite
group \(G\), then there is an isomorphism of \(G\)-representations \(\RR (X_1 \times X_2) = \RR X_1 \otimes \RR X_2\) (for essentially the same reason that the basic tensors \(e_i \otimes e_j\) for \(i = 1,\dots, m\), \(j=1,\dots n\) form a basis for \(\RR^m \otimes \RR^n\)). It follows
that \( (\RR^n)^{\otimes l} = \RR (\{1, \dots, n\}^l) \) is the permutation
representation associated to the set \(\{1, \dots, n\}^l\) with the ``diagonal''
\(\Sigma_n\) action. 
\begin{lemma}[{cf. \cite[\S 1]{jonesPottsModelSymmetric1994}, \cite[\S 5.2]{benkartPartitionAlgebrasInvariant2017}}]
    \label{lem:orbit-sums}
    Let \(G \) be a finite group and let \(X \) be a finite set with a
    \(G\)-action. Then there is a natural vector space isomorphism \(\RR (X/G) \xrightarrow{\simeq} (\RR X)^G\),
    where $X/G$ are the orbits of $X$, defined as follows: for each orbit \(Gx \in X/G\), send the
    basis vector \(e_{Gx}\) to \(\sum_{x'\in Gx} e_{x'}\).
\end{lemma}
In particular, to compute \(( (\RR^n)^{\otimes l}
)^{\Sigma_n}\), and thus by \cref{cor:hom-tensor} the spaces of permutation
equivariant linear layers \((\RR^n)^{\otimes m} \to (\RR^n)^{\otimes m'}\) for all \(m+
m' = l\), all we need to do is compute the orbits of the diagonal \(\Sigma_n\)
action on \(\{1, \dots, n\}^l\) --- and that is precisely the content of \cref{lem:bell-nos}.

\section{Proofs}

\begin{proof}[Proof of \cref{lem:hom-tensor}]
    In general, for any two real vector spaces \(V, W \) there is an isomorphism \( V^\vee \otimes W \simeq \Hom(V, W) \) sending a basic tensor \( \lambda \otimes w \) to the linear map \(\varphi: V \to W \) defined by \(\varphi(v) = \lambda(v)\cdot w \). In our case, this shows that 
    \begin{equation}
        \Hom((\RR^n)^{\otimes m}, (\RR^n)^{\otimes m'}) \simeq \big( (\RR^n)^{\otimes m} \big)^\vee \otimes  (\RR^n)^{\otimes m'}.
    \end{equation}
    As \((\RR^n)^{\otimes m} \) is a real representation of the finite group \(\Sigma_n \) it admits an equivariant Euclidean inner product: the natural, explicit one to use is simply defined on standard basis tensors as
    \begin{equation}
        \langle e_{i_1} \otimes \cdots \otimes e_{i_m},  e_{j_1} \otimes \cdots \otimes e_{j_m}\rangle = \prod_{k=1}^m \delta_{i_k j_k}.
    \end{equation}
    Such an inner product is equivalent to an isomorphism \(  (\RR^n)^{\otimes m} \simeq \big( (\RR^n)^{\otimes m} \big)^\vee \) by the map \( v \mapsto \langle v, - \rangle \). Hence  
    \begin{equation}
        \Hom((\RR^n)^{\otimes m}, (\RR^n)^{\otimes m'}) \simeq  (\RR^n)^{\otimes m}  \otimes  (\RR^n)^{\otimes m'} = (\RR^n)^{\otimes (m+m')}
    \end{equation}
    and taking invariants on both sides completes the proof.
\end{proof}

\begin{proof}[Proof of \cref{cor:hom-tensor}]
    \Cref{lem:hom-tensor} gives us an invertible map from 
    \(\big((\RR^n)^{\otimes (m+m')}\big)^{\Sigma_n}\) to \(\Hom((\RR^n)^{\otimes m}, (\RR^n)^{\otimes m'})^{\Sigma_n}\). Then the parameterization basis of \((\RR^n)^{\otimes l}\) is mapped to a parametrization basis of \(\Hom((\RR^n)^{\otimes m}, (\RR^n)^{\otimes m'})^{\Sigma_n}\).
\end{proof}

\begin{proof}[Proof of \cref{lem:orbit-sums}]
     We note that the lemma is essentially \cite[Ex. 2.6 (a)]{serreLinearRepresentationsFinite1977}. A \(\sum_{x \in X} c_x x \in \RR X \) lies in the invariant subspace \( (\RR X)^G \) if and only if for each \(g \in G \)
     \begin{equation}
         \sum_{x \in X} c_x x = g \cdot \sum_{x \in X} c_x x = \sum_{x \in X} c_x gx = \sum_{x \in X} c_{g^{-1} x} x
     \end{equation} 
     where in the last step we have reindexed the sum. The condition that \(c_x = c_{g^{-1} x} \) for all \(g \in G, x\in X \) says precisely that the coefficients \(c_x \) are constant on \(G \)-orbits. In other words, if \(c_{Gx} \) is the common value of the \(c_{x'}\) as \(x' \) runs over the orbit \(Gx \) then 
     \begin{equation}
         \sum_{x \in X} c_x x = \sum_{Gx \in X/G} c_{Gx} \sum_{x' \in Gx} e_{x'}
     \end{equation}
     showing that the map defined in \cref{lem:orbit-sums} is surjective. It is clearly injective, since for example the orbit sums \(\sum_{x' \in Gx} e_{x'} \) are pairwise orthogonal.
\end{proof}

\begin{proof}[Proof of \cref{lem:bell-nos}]
Since the proof of \cref{lem:bell-nos} for any \(l \) is identical to its proof in the \(l\)-even case considered by prior work on partition algebras, we omit a proof and instead provide references where the \(l\)-even case is discussed: see \cite[\S 1]{jonesPottsModelSymmetric1994}, \cite[\S 3]{halversonPartitionAlgebras2005} and \cite[\S 5.1]{benkartPartitionAlgebrasInvariant2017}.
\end{proof}

\begin{proof}[Proof of \cref{lem:algebra}]
     We note that in the essence of this proof is the same as the proof that a product of symmetric polynomials is symmetric. 
     
     Tensor (i.e. Kronecker) product defines a bilinear map 
     \begin{equation}
          (\RR^n)^{\otimes p}  \times  (\RR^n)^{\otimes p'} \to  (\RR^n)^{\otimes l} 
     \end{equation}
     sending \( (v, w) \) to \(v \otimes w \). Taking \(\Sigma_n\)-invariants gives a map
     \begin{equation}
         \big( (\RR^n)^{\otimes p}  \times  (\RR^n)^{\otimes p'}  \big)^{\Sigma_n} \to \big( (\RR^n)^{\otimes l} \big)^{\Sigma_n}.
     \end{equation}
     Finally there is an inclusion \(\big( (\RR^n)^{\otimes p} \big)^{\Sigma_n} \times \big( (\RR^n)^{\otimes p'} \big)^{\Sigma_n} \subseteq \big( (\RR^n)^{\otimes p}  \times  (\RR^n)^{\otimes p'}  \big)^{\Sigma_n}  \) since the left hand side consists of pairs of tensors \((v, w) \) invariant to the action of independent permutations \(\sigma, \tau \in \Sigma_n \) as \((\sigma \cdot v, \tau \cdot w) \), and this condition is stronger than invariance to the diagonal action of a single permutation \(\sigma \) as  \((\sigma \cdot v, \sigma \cdot w) \).
\end{proof}

\begin{lemma}
    \label{lem:poset-thing}
    Let \(V\) be a real vector space and let 
    \begin{equation}
        \{v_\alpha \in V \, | \, \alpha \in \cA\}
    \end{equation}
    be a set of vectors in \(V\) indexed by a partially ordered set \(\cA\).
    Suppose that for every \(\alpha \in A\) there exists a linear functional
    \(\lambda_\alpha: V \to \RR\) with the property that 
    \begin{equation}
        \lambda_\alpha(v_\alpha) \neq 0 \text{  and  } \lambda_\alpha(v_\beta) = 0 \text{  unless  } \beta \preceq \alpha,
    \end{equation}
    where \(\preceq\) denotes the partial order on \(\cA\). Then, \(\{v_\alpha
    \in V \, | \, \alpha \in \cA\}\) is linearly independent.
\end{lemma}

\begin{proof}
    Consider an equation of the form
    \begin{equation}
        \label{eq:lin-dep}
        \sum_{\alpha \in \cA} c_\alpha v_\alpha = 0
    \end{equation}
    where all but finitely many of the \(c_\alpha\) are \(0\). Suppose towards
    contradiction that some \(c_\alpha\neq 0\), and let 
    \begin{equation}
        \cB = \{\alpha \in \cA \, | \, c_\alpha \neq 0 \} \subseteq \cA.
    \end{equation}
    By hypothesis, \(\cB\) is a nonempty finite partially ordered set, and as
    such it has at least one minimal element, i.e. an \(\alpha^* \in \cB\) such
    that if \(\alpha \in \cB\) and \(\alpha \preceq \alpha^*\) then \(\alpha =
    \alpha^*\). It follows that \(\lambda_{\alpha^*} (v_\alpha) = 0\) for \(\alpha \in
    \cB \setminus \{\alpha^*\}\), hence applying \(\lambda_{\alpha^*}\) to
    \cref{eq:lin-dep} results in 
    \begin{equation}
        c_{\alpha^*} \lambda_{\alpha^*}(v_{\alpha^*}) = 0
    \end{equation}
    and thus \(c_{\alpha^*} = 0\) (since \(\lambda_{\alpha^*}(v_{\alpha^*}) \neq 0\)), a contradiction.
\end{proof}

\begin{lemma}
\label{lem:obs2}
    For any set partition \(\cP \) of \(\{1, \dots, l\}\) the vector \(d_{\cP}\) is the sum of the \(e_{J} = e_{j_1} \otimes \cdots
    \otimes e_{j_l} \) over all indices \( J = (j_1,\dots, j_l)\) that are constant on every set occurring in the partition \(\cP \) --- equivalently, \(d_{\cP} = \sum_{\cP \preceq \cQ} e_{\cQ}\) where \( \cQ \preceq \cP \) if and only if \(\cQ \) refines the partition \(\cP \).
\end{lemma}

\begin{proof}
    After permuting the \(l\) tensor factors, we may reduce to the case where \(\cP \) has the form 
   \begin{equation}
       \cP = \coprod_{i=1}^n \{ \sum_{j<i} p_j + 1, \dots, \sum_{j<i} p_j + p_i  \}
   \end{equation}
   where \(p_1,\dots, p_n \in \NN \) and \(\sum_{i=1}^n p_i = l \). Then by direct calculation
   \begin{equation}
       d_{\cP} = \sum_{j_1, \dots, j_n = 1}^n e_{\underbrace{j_1, \dots, j_1}_{p_1 \text{times}}} \otimes e_{\underbrace{j_2, \dots, j_2}_{p_2 \text{times}}} \otimes \cdots \otimes e_{\underbrace{j_n, \dots, j_n}_{p_n \text{times}}},
   \end{equation}
   and evidently the indices \( (\underbrace{j_1, \dots, j_1}_{p_1 \text{times}}, \dots,  \underbrace{j_n, \dots, j_n}_{p_n \text{times}})\) occurring in the sum are exactly those constant on each of the sets \(\{ \sum_{j<i} p_j + 1, \dots, \sum_{j<i} p_j + p_i  \}\).
\end{proof}

\begin{proof}[Proof of \cref{thm:alt-basis}]
    Since both the \(e_{\cP}\) and the \(d_{\cP}\) are indexed by set
    partitions of \(\{1, \dots, l\}\) into at most \(n \) non-empty subsets they have the same cardinality, and we
    already know the \(e_{\cP}\) are a basis for \(\big( (\RR^n)^{\otimes l}
    \big)^{\Sigma_n}\). Thus, it will suffice to show the  \(d_{\cP}\) are linearly independent. 
    We will prove this by exhibiting a set of linear functionals
    \begin{equation}
        \{\lambda_{\cP}: \big( (\RR^n)^{\otimes l} \big)^{\Sigma_n} \to \RR \, | \, \cP \text{ is a set partition of } \{1, \dots, l\}\}
    \end{equation}
    such that 
    \begin{equation}
        \label{eq:lambda-prop}
        \lambda_{\cP} (d_{\cP}) \neq 0 \text{  and  } \lambda_{\cP}(d_{\cQ}) = 0 \text{  unless } \cQ \preceq \cP,
    \end{equation}
    where \( \cQ \preceq \cP \) if and only if \(\cQ \) refines the partition \(\cP \),
    satisfying the properties of \cref{lem:poset-thing} below. Explicitly, for
    each partition \(\cP\) choose an index \(I_\cP = (i_1, \dots, i_l) \in \{1,
    \dots, n\}^l\) such that \(\cP\) is the partition associated to \(I_{\cP}\)
    as in \cref{lem:bell-nos}, and let \(e_{I_{\cP}} = e_{i_1} \otimes \cdots
    \otimes e_{i_l} \in (\RR^n)^{\otimes l}\). Equivalently, \(e_{I_{\cP}}\) is
    one of the standard basis vectors for \((\RR^n)^{\otimes l}\) occurring in
    the orbit sum defining \(e_{\cP}\). Then define \(\lambda_{\cP}\) as dot
    product with \(e_{I_{\cP}}\):
    \begin{equation}
        \label{eq:our-lambda}
        \lambda_{\cP}(v) = \langle e_{I_{\cP}}, v \rangle \text{  for  } v \in \big( (\RR^n)^{\otimes l} \big)^{\Sigma_n}.
    \end{equation}
    
    First, since by definition \( I_{\cP} \) is constant on each set of the partition \(\cP \), our characterization of \(d_{\cP}\) in \cref{lem:obs2} gives 
    \( \langle e_{I_{\cP}}, d_{\cP} \rangle = 1 \). Next, if \(\cQ \not \preceq \cP \), then since \(\cQ \) doesn't refine \(\cP \) writing 
    \begin{equation}
        \cP = \{S_1,\dots, S_n\} \text{  and  } \cQ = \{T_1,\dots, T_n\}
    \end{equation}
    there must be a non-empty \(T_i \) such that \( T_i \not \subseteq S_j \) for all \(j \). Since \(T_i \subseteq \bigcup_j S_j \), there must be distinct \(S_j, S_k \in \cP \) with \(S_j \cap T_i \neq \emptyset, S_k \cap T_i \neq \emptyset \). By design the index \(I_{\cP} \) takes distinct values on \(S_j \) and \(S_k \), but for every  \(e_{J} = e_{j_1} \otimes \cdots
    \otimes e_{j_l} \) occurring in \(d_{\cQ}\) with non-zero coefficient the index \(J\) is constant on \(T_i \). Thus  \(\langle e_{I_{\cP}}, d_{\cQ} \rangle = 0\).
\end{proof}

\begin{proof}[Proof of \cref{prop:diagram}]
     This follows from \cref{lem:obs2} and \cite[p.
     263]{jonesPottsModelSymmetric1994} (see also \cite[\S 4.2-3]{benkartPartitionAlgebrasInvariant2017}).
\end{proof}

\section{An algorithm for multiplication with the tensors \(d_{\cP}\)}
\label{sec:algorithm}

In this section we describe how to apply one of the tensors \(d_{\cP} \in  \big( (\RR^n)^{\otimes (m+m')} \big)^{\Sigma_n} \simeq \Hom((\RR^n)^{\otimes m}, (\RR^n)^{\otimes m'})^{\Sigma_n}\) appearing in \cref{thm:alt-basis} to a tensor \(v \in  (\RR^n)^{\otimes m}\). We will work in the usual bases for these tensor products, and use notation of the form
\begin{equation}
    v = \sum_{i_1, \dots, i_m=1}^n v_{i_1 \dots i_m} e_{i_1} \otimes \cdots \otimes e_{i_m}.
\end{equation}
 Given another tensor  
\begin{equation}
    w = \sum_{j_1, \dots, j_l=1}^n w_{j_1 \dots j_l} e_{j_1} \otimes \cdots \otimes e_{j_l} \in (\RR^n)^{\otimes l}
\end{equation}
and ordered tuples of indices $T$ and $S$ with
\(\set(T) \subseteq \{1,\dots, m \}, \set(S) \subseteq \{1,\dots, l\} \) 
and tuple length \(d \leq \min \{m, l \}\),
we will use the notation \(\tensordot(v, w, (S, T))  \) to denote the tensor
contraction operation implemented in PyTorch as \( \texttt{tensordot}(v, w,
\texttt{dims}=(S, T)) \) \citep{torch}. The special case where \(T = (m, m-1,
\dots, m-d+1), S = (1,\dots, d ) \) will be abbreviated as \(  \tensordot(v, w,
d) \). We refer to the documentation at
\href{https://pytorch.org/docs/stable/index.html}{https://pytorch.org/docs/stable/index.html}
for further details.

Unravelling \cref{lem:hom-tensor}, we see that our goal is to calculate \(\tensordot(v, d_{\cP}, m)\). The point we wish to make is that for many partitions \(\cP \) this can be accomplished with multiple contractions over fewer than \(m \) indices, and moreover the contractions occurring can be replaced with sums and indexing operations.

The contraction \(\tensordot(v, d_{\cP}, m) \) is invariant to permutations of the tensor factors of \(v\) and the first \(m \) tensor factors of \( d_{\cP} \). Moreover, we are free to permute the last \(m' \) factors of \(d_{\cP} \) provided we apply the inverse permutation to the \(m'\) factors of \(\tensordot(v, d_{\cP}, m) \). These two observations allow us to reduce to the case where \(\cP \) has the form 
\begin{equation}
(\coprod_{i=1}^{a} S_i) \coprod (\coprod_{i=1}^b T_i) \coprod (\coprod_{i=1}^c B_i ) 
\end{equation}
where
\begin{itemize}
    \item \(a + b + c =n \),
    \item the \(S_i \) are consecutive and contiguous sets of indices in \(\{1,\dots, m \} \) beginning at \(1 \) and ending at say \(p \), 
    \item  the \(B_i \) are consecutive and contiguous sets of indices in \(\{m+1,\dots, m + m' \} \) beginning at say \(p+ p'\) and ending at \(m + m' \), and 
    \item the \(T_i \) partition \(\{p, p+1,\dots, p+ p'\} \), and moreover they decompose as \(T_i = T'_i + T''_i \) where \(T'_1,\dots, T'_b\) are consecutive and contiguous sets of indices in \(\{p+1,\dots, m \} \) and \(T''_b,\dots, T''_1\) are consecutive and contiguous sets of indices in \(\{p+p'+1,\dots, m+m' \} \). 
\end{itemize}
 It follows that
\begin{equation}
    \label{eq:happy-factors}
     \begin{split}
        d_{\cP} &= e_{\{\{1,\dots, \nrm{S_1}\}\}} \otimes \cdots \otimes e_{\{\{1,\dots, \nrm{S_a}\}\}} \\
        &\otimes \Psi(e_{\{\{1,\dots, \nrm{T_1}\}\}} \otimes \cdots \otimes e_{\{\{1,\dots, \nrm{T_b}\}\}}) \\
        &\otimes e_{\{\{1,\dots, \nrm{B_1}\}\}} \otimes \cdots \otimes e_{\{\{1,\dots, \nrm{B_c}\}\}}
     \end{split}
\end{equation}
where \(\Psi \) permutes tensor factors according to a certain permutation of \(\{p, p+1,\dots, p+ p'\} \) (the one separating each \(T_i \) into the subsets \(T'_i, T''_i\)). 

We now make repeated use of two simple calculations: first, 
\begin{equation}
\label{eq:extract-factor}
     \tensordot(v, w, d) =  \tensordot(v, w', d) \otimes w'' \text{whenever \(w = w' \otimes w'' \) where \(w' \) has \( \geq d \) indices}.
\end{equation}
That is, \(w'' \) can be extracted from the \(\tensordot \). On the other hand,
\begin{equation}
    \label{eq:1atatime}
    \tensordot(v, w, d) =  \tensordot(\tensordot(v, w', d'), w'', d-d') \text{whenever \(w = w' \otimes w'' \) where \(w' \) has \( d ' \leq d \) indices}
\end{equation}
In particular, \cref{eq:extract-factor} applies to the factors \(e_{B_i} \) in \cref{eq:happy-factors}, giving
\begin{equation}
    \begin{split}
        &\tensordot(v, d_{\cP}, m) \\
        &= \tensordot(v,  e_{\{\{1,\dots, \nrm{S_1}\}\}} \otimes \cdots \otimes e_{\{\{1,\dots, \nrm{S_a}\}\}} \otimes \Psi(e_{\{\{1,\dots, \nrm{T_1}\}\}} \otimes \cdots \otimes e_{\{\{1,\dots, \nrm{T_b}\}\}})  , m) \\
        &\otimes e_{\{\{1,\dots, \nrm{B_1}\}\}} \otimes \cdots \otimes e_{\{\{1,\dots, \nrm{B_c}\}\}}.
    \end{split}
\end{equation}
An explicit calculation shows that the tensor operation \(x \mapsto x \otimes e_{\{1,\dots, p\}}\) simply creates a tensor with \(p \) more indices than \(x \) and places copies of \(x \) along indices of the form \( \dots, \underbrace{i, i,\dots, i}_{p \text{ times}}\), with zeros elsewhere. Thus once \[\tensordot(v,  e_{\{\{1,\dots, \nrm{S_1}\}\}} \otimes \cdots \otimes e_{\{\{1,\dots, \nrm{S_a}\}\}} \otimes \Psi(e_{\{\{1,\dots, \nrm{T_1}\}\}} \otimes \cdots \otimes e_{\{\{1,\dots, \nrm{T_b}\}\}})  , m) \] is computed, multiplication with the \(e_{\{\{1,\dots, \nrm{B_i}\}\}} \) can be accomplished efficiently with indexing operations.\footnote{Observe that in contrast Kronecker multiplication with a general tensor \(w'' \) with \(p \) entries would require \(n^p \) scalar multiplications per entry of \(x \), so e.g. if \(x \) has \(q \) indices \(n^{p+q} \) scalar multiplications in total.}

\Cref{eq:1atatime} shows that contraction with the \(e_{\{\{1,\dots, \nrm{S_i}\}\}} \) can be carried out ``one \(i \) at a time,'' from left to right. Based on our conventions for \(\tensordot \), a contraction \(\tensordot(v, e_{\{\{1,\dots, p\}\}}, p)  \) yields a tensor with \(m - p \) indices, with \((i_1, \dots, i_{m-p})\)-th entry 
\begin{equation}
    \sum_{j=1}^n v_{i_1 \dots i_{m-p} \underbrace{j\dots j}_{p \text{ times}}}
\end{equation}
Hence these contractions can be implemented with index operations and summation.\footnote{Observe that \cref{eq:1atatime} reduces the cost of tensor contraction from roughly \(n^d \) to \(n^{d'} + n^{d-d'}\) multiplications and additions. In the special case where \(w' = e_{\{1,\dots, d'\}}\) the term \(n^{d'} \) effectively drops to \(n \).}

Finally, the tensor \(  \Psi(e_{T_1} \otimes \cdots \otimes e_{T_b})  \) must be dealt with. We claim that \(\tensordot \) with this tensor is a \emph{transfer operation} as described in \cite[\S 5]{panPermutationEquivariantLayers2022}. Indeed, they define transfer operations as those corresponding to sets in the partition \(\cP \) having non-empty intersection with both \(\{1,\dots, m\} \) and \(\{m+1,\dots, m+m' \}\), which is exactly the role played by the \(T_i \). Since we have already addressed how to multiply with the \(e_{\{\{1,\dots, \nrm{S_i}\}\}}\) and \(e_{\{\{1,\dots, \nrm{B_i}\}\}} \) we may as well assume for simplicity that \(d_{\cP} =  \Psi(e_{\{\{1,\dots, \nrm{T_1}\}\}} \otimes \cdots \otimes e_{\{\{1,\dots, \nrm{T_b}\}\}}) \). Then,
\begin{equation}
    \tensordot(v, \Psi(e_{\{\{1,\dots, \nrm{T_1}\}\}} \otimes \cdots \otimes e_{\{\{1,\dots, \nrm{T_b}\}\}}) )_{\underbrace{i_1\dots i_1}_{\nrm{T''_1} \text{ times}} \dots \underbrace{i_b\dots i_b}_{\nrm{T''_b} \text{ times}}} = v _{\underbrace{i_1\dots i_1}_{\nrm{T'_1} \text{ times}} \dots \underbrace{i_b\dots i_b}_{\nrm{T'_b} \text{ times}}}. 
\end{equation}
Clearly, this is essentially an indexing operation.\footnote{In other words, rather than performing \(n^{\sum \nrm{T'_i}} \) multiplications and additions, we are just copying arrays.}

\section{Examples}
\label{sec:examples}
We look at the case where \(m = m' = 1 \) and \(n \geq 2 \). Here the space of
permutation equivariant linear layers is \(\big((\RR^n)^{\otimes 2} \big)^{\Sigma_n}\),
i.e. matrices invariant under simultaneous permutation of rows and columns, and
these corresponding to equivariant maps \(\RR^n \to \RR^n\) by matrix-vector
multiplication. 

\subsection{Partitions associated with index tuples}
The simplest case of \cref{lem:bell-nos} occurs when \(l = 2 \). Here, given \( (i_1, i_2) \in \{1,\dots, n\}^2 \) the associated set partition is 
\begin{equation}
    \Pi(i_1, i_2) = \begin{cases}
    \{ \{1,2\} \} & \text{ if } i_1 = i_2 \\
    \{ \{1\},\{2\} \} & \text{ if } i_1 \neq i_2 
    \end{cases}
\end{equation}
Let \(\cP_1 = \{ \{1,2\} \}\) and \(\cP_2 = \{ \{1\},\{2\} \}\). Then viewing tensors in \(\RR^n \otimes \RR^n \) as \(n \times n \) matrices, 
\begin{equation}
    \begin{split}
        e_{\cP_1} &= \sum_i e_i \otimes e_i = I_n \text{  and} \\
        e_{\cP_2} &= \sum_{i \neq j} e_i \otimes e_j =  \mathbf{1}\mathbf{1}^T - I 
    \end{split}
\end{equation}
where \(I\) is the $n \times n$ identity matrix and \(\mathbf{1} = [1,\dots, 1]^T\) the \(1\)s
vector.

\subsection{Ranks of basis vectors}
Observe that \(\mathbf{1}\mathbf{1}^T - I\) is \emph{not} a Kronecker
product \(uv^T \). Indeed, 
it has rank \(n \),
since $ det(-I + \mathbf{1}\mathbf{1}^T) =  (1+\mathbf{1}^T (-I) \mathbf{1}) det(I) = (1-n)n \neq 0$, 
whereas any \(uv^T \) has rank 1. 

On the other hand, the
basis of \cref{thm:alt-basis} is
\begin{equation}
    d_{\cP_1} = I \text{  and  } d_{\cP_2} = \mathbf{1}\mathbf{1}^T 
\end{equation}
and \(\mathbf{1}\mathbf{1}^T\) is of course rank 1. This is of course the basis
used in DeepSets \citep{zaheerDeepSets2018} (\(I \) is the identity map,
\(\mathbf{1}\mathbf{1}^T \) is the vector sum multiplied by \(\mathbf{1}\)).

\end{document}